\documentclass{IEEEtran}

\IEEEoverridecommandlockouts                              

\usepackage{amsthm}
\usepackage{times}
\usepackage{multicol}
\usepackage[bookmarks=true]{hyperref}
\usepackage{xcolor}
\usepackage{hyperref}
\usepackage{amsmath, amssymb}
\usepackage{amsfonts}
\usepackage{graphicx}
\usepackage{siunitx}
\usepackage{standalone}
\usepackage{booktabs}
\usepackage[ruled,vlined,linesnumbered,noend]{algorithm2e}
\usepackage{mdframed}
\usepackage{fancyvrb,multirow}
\usepackage{soul}
\usepackage{dsfont,mathabx}
\usepackage{array, booktabs}
\usepackage{subfigure}
\usepackage{booktabs}
\usepackage{makecell}
\usepackage{threeparttable}
\usepackage{float}
\usepackage{stfloats}
\usepackage{mathrsfs}

\newtheorem{theorem}{Theorem}

\newtheorem{lemma}[theorem]{Lemma}
\theoremstyle{definition}

\newtheorem{definition}{Definition}
\newtheorem{remark}{Remark}

\title{\LARGE \bf
Customize Harmonic Potential Fields via Hybrid\\
Optimization over Homotopic Paths}

\author{Shuaikang Wang, Tiecheng Guo and Meng Guo
  \thanks{The authors are with the College of Engineering,
    Peking University, Beijing 100871, China.
    This work was supported by the National Natural Science Foundation
    of China (NSFC) under grants 62203017, T2121002, U2241214;
    and by the Fundamental Research Funds for the central universities.
    Contact: {\tt\small meng.guo@pku.edu.cn}.}
}

\begin{document}
\maketitle
\thispagestyle{empty}
\pagestyle{empty}


\begin{abstract}
  Safe navigation within a workspace is a fundamental skill
  for autonomous robots to accomplish more complex tasks.
  Harmonic potentials are artificial potential fields that are
  analytical, globally convergent and provably free of local minima.
  Thus, it has been widely used for generating safe and reliable
  robot navigation control policies.
  However, most existing methods do not allow customization of the
  harmonic potential fields nor the resulting paths,
  particularly regarding their topological properties.
  In this paper, we propose a novel method
  that automatically finds homotopy classes of paths
  that can be generated by valid harmonic potential fields.
  The considered complex workspaces can be as general as
  forest worlds consisting of numerous overlapping star-obstacles.
  The method is based on a hybrid optimization algorithm that
  searches over homotopy classes,
  selects the structure of each tree-of-stars within the forest,
  and optimizes over the continuous weight parameters for each purged tree
  via the projected gradient descent.
  The key insight is to transform the forest world
  to the unbounded point world via proper diffeomorphic transformations.
  It not only facilitates a simpler design of the multi-directional D-signature
  between non-homotopic paths,
  but also retain the safety and convergence properties.
  Extensive simulations and hardware experiments are conducted for
  non-trivial scenarios,
  where the navigation potentials are customized
  for desired homotopic properties.
  Project page:~\href{https://shuaikang-wang.github.io/CustFields}{\url{https://shuaikang-wang.github.io/CustFields}}.
\end{abstract}

\section{Introduction}\label{sec:intro}
Safe navigation within a given complex workspace is essential for
autonomous robots to accomplish more tasks,
i.e., to drive the robot from an initial state to the target state
while staying within the allowed workspace and avoiding collision with obstacles.
As a long-standing research field,
many powerful methods have been proposed,
such as optimal control theory~\cite{laumond1998robot},
probabilistic sampling~\cite{lavalle2006planning},
potential fields~\cite{rimon1992exact, vlantis2018robot,
  loizou2017navigation, rousseas2024reactive}
and so on.
These methods have been successfully applied to different dynamic systems
including mobile robots and manipulators,
for various applications such as navigation and manipulation.
Moreover, in many cases, the end-users might have preferences
over the final paths such as their topological properties,
i.e., the sequence of gates between obstacles that the path should (or not) pass.
For instance, a service robot might navigate through as many rooms
during surveillance tasks;
might not pass the meeting room with an ongoing conference;
and should avoid the kitchen when taking the trash out.
However, it remains a challenge to design a navigation policy
that can both ensure safety and convergence,
while accommodating for such user customizations.
\begin{figure}[t!]
  \centering
  \includegraphics[width=0.98\hsize]{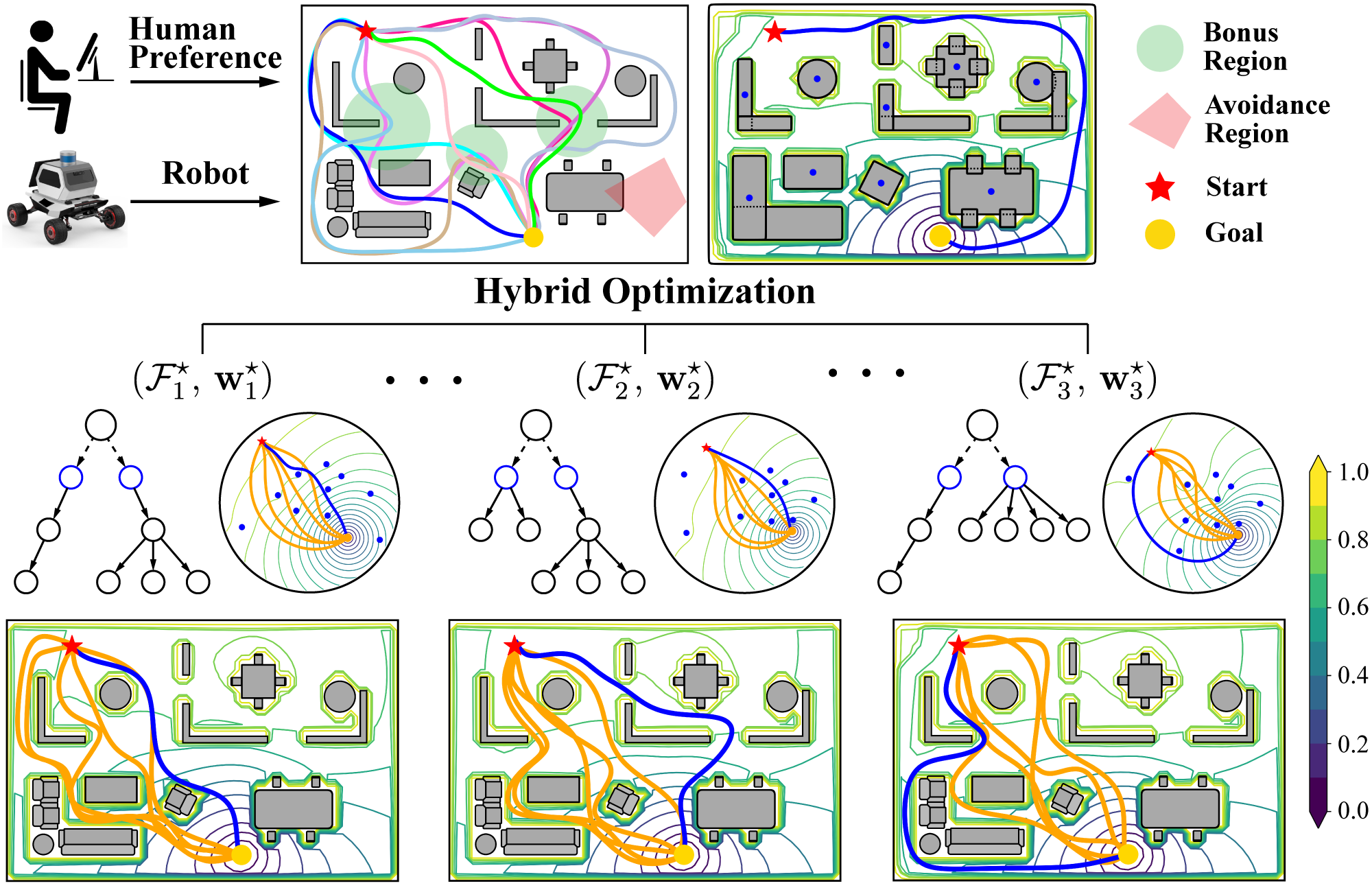}
  \vspace{-0.1in}
  \caption{The proposed hybrid framework for customizing
    harmonic potentials, where the forest structure
    and weight parameters are optimized (\textbf{second row}).
    The resulting paths in the derived potentials
      (\textbf{third row}) can be selected
   given the desired topological properties (\textbf{first figure}),
    compared with the default harmonic potential and the
    resulting path (\textbf{second figure}).}
  \label{fig:hybrid}
  \vspace{-0.15in}
\end{figure}
\subsection{Related Work}\label{subsec:intro-related}
As the most relevant to this work,
navigation functions (NF) pioneered by Rimon and
Koditschek~\cite{rimon1992exact}
provide an intuitive yet powerful framework for motion planning.
Compared to the traditional artificial potential field (APF) methods~\cite{lavalle2006planning, warren1989global, warren1990multiple},
which rely on attraction fields near the goal and repulsion fields near obstacles
but suffer from the local minimum,
the NFs are proven to possess the goal
as the sole global minimum with a set of isolated saddle points of measure zero.
This framework is further expanded to harmonic potential fields
as proposed in~\cite{vlantis2018robot, rousseas2024reactive, loizou2017navigation}.
Via a novel diffeomorphic transformation to \emph{point worlds},
harmonic potentials are constructed more efficiently without parameter tuning,
while retaining safety and global convergence.
These methods have been applied successfully to various dynamic
systems~\cite{loizou2017navigation, dimarogonas2006feedback}
and different workspace or obstacle representations~\cite{vlantis2018robot, li2018navigation}.
Recent work in~\cite{rousseas2022optimal,rousseas2022trajectory}
adopts a wider set of harmonic basis functions to address more complex workspaces,
which however requires solving parameterized optimizations online,
instead of providing analytic solutions.
Partially-unknown workspaces are considered in~\cite{rousseas2022trajectory,
vasilopoulos2022reactive, arslan2019sensor, wang2025hybrid, dahlin2023creating},
via recursive adaptation of the underlying potentials
and paths given online measurements.
The generated potentials have been combined with numerous
feedback controllers~\cite{loizou2017navigation, dahlin2023obstacle} for online execution.
However, most of these methods focus on the convergence and safety
of the resulting paths,
while neglecting how the harmonic potentials can be customized
to generate paths with different topological properties.

On the other hand, homotopic paths are paths that connect the same start and goal points,
and can be continuously deformed into each other,
without leaving the workspace or intersecting any obstacles.
It incorporates the topological properties of the paths within the
obstacle-cluttered workspace, in addition to the commonly-used geometric properties
such as length and smoothness.
It has been widely adopted as a classification of paths in robot
navigation~\cite{zhou2020robust,bhattacharya2012topological}
and object manipulation~\cite{huang2024homotopic}.
Paths within the same homotopy class are considered \emph{close} in the
cluttered environments, and paths in different classes are considered
far apart.
Thus, homotopy distances have been used to measure the similarity between paths,
and further guide the path planning process to generate paths with diverse
homotopic properties, e.g.,
in~\cite{zhou2020robust} for trajectory optimization of aerial vehicles to escape local minima;
in~\cite{bhattacharya2012topological}
for search-based planning to incorporate topological information;
in~\cite{huang2024homotopic} for the generation
of parallel paths to handle deformable objects and connected robotic fleets.
However, there has been limited research on
the customization of potential fields to generate paths
with desired homotopic properties, nor
to find all homotopy classes of paths that can be generated
by valid potential fields.

\subsection{Our Method}\label{subsec:intro-our}
This work tackles the customization of harmonic potential fields,
to generate paths with desired topological properties.
As shown in Fig.~\ref{fig:hybrid}, the proposed hybrid optimization algorithm,
consists of three interleaved layers:
(I) simplifies the representation for homotopy classes from
the forest world to the point world, via diffeomorphic transformation;
(II) selects different structures for each tree-of-stars within the forest world,
including the roots and parent-child relationships;
and (III) optimizes the continuous weight parameters of each purged tree,
via approximated gradient descent over
a multi-directional distance metric, subject to feasibility constraints.
The derived harmonic potential can be used to guide a robot as navigation policies,
by which the resulting paths are validated for safety, convergence and
topological properties.
The homotopy properties of the harmonic potentials can be customized
through user specification of a D-signature or
automatically given the prioritized regions.

Main contribution of this work is threefold:
(I) A new problem as the customization of harmonic potential fields is introduced,
according to the homotopic properties of the resulting paths;
(II) Theoretical results regarding the invariance of homology under diffeomorphic transformation
are derived formally;
(III) The hybrid optimization algorithm is proposed that can automatically
generate valid harmonic potentials for desired homotopy properties.
To our knowledge, {this is the first work} that provides such results.

\section{Preliminaries}\label{sec:preliminary}
\subsection{Diffeomorphic Transformation and Harmonic Potentials}\label{subsec:diff-transform}
Consider a \emph{forest world}~$\mathcal{F}\subset\mathbb{R}^2$,
which has an outer boundary of the workspace~$\mathcal{W} \subset \mathbb{R}^2$
and inner boundaries of~$M$ \emph{tree-of-stars} obstacles~$\mathcal{O}_i\subset \mathbb{R}^2$
for $i\in \mathcal{M}\triangleq \{1,2,\cdots,M\}$.
The forest world $\mathcal{F}$ can be mapped to a point
world~$\mathcal{P}\triangleq \mathbb{R}^2 \backslash \bigcup_{i=1}^M\,P_i$
by a diffeomorphic transformation~$\Phi_{\mathcal{F}\rightarrow \mathcal{P}}:
\mathcal{W} \rightarrow \mathcal{P}$,
such that~$\Phi_{\mathcal{F}\rightarrow \mathcal{P}}(p)\in \mathcal{P}$,
$\forall p \in \mathcal{F}$; and
$\Phi_{\mathcal{F}\rightarrow \mathcal{P}}(o_i)\triangleq P_i$,
$\forall o_i \in \mathcal{O}_i$ and $\forall i\in \mathcal{M}$.

\begin{definition}[Harmonic Potential]\label{def:homotopy-path}
  The \emph{parameterized harmonic potential}
  function~$\phi_{\mathcal{P}}:\mathcal{P}\rightarrow \mathbb{R}^+$
  is defined as:
\begin{equation}\label{eq:harmonic-point-potential}
  \phi_{\mathcal{P}}(q,\,\mathbf{w}) \triangleq w_{\texttt{g}}\, \phi(q,\, q_{\texttt{g}})
  - \sum_{i=1}^M w_{i}\,\phi(q,\, P_i),
\end{equation}
where~$\phi:\mathbb{R}^2 \times \mathbb{R}^2\rightarrow \mathbb{R}$
is the harmonic term:
$\phi(q,\,P) \triangleq \ln\left(\|q-P\|^2\right)$
for $q,\,P\in \mathbb{R}^2$;
$\phi(q,\, q_{\texttt{g}})$ is the potential for the
goal~$q_{\texttt{g}}\in \mathcal{P}$ and $\phi(q,\, P_i)$ for the
point~$P_i\in \mathbb{R}^2$ and~$i\in \mathcal{M}$;
$\mathbf{w}\triangleq [w_{\texttt{g}},w_1,\cdots,w_M]^\intercal > 0$
are weight parameters.
\hfill $\blacksquare$
\end{definition}

Then, the harmonic potentials within the actual forest
world~$\mathcal{F}$ can be constructed as:
\begin{equation}\label{eq:complete-nf}
  \varphi_{\mathcal{F}}(p) \triangleq \sigma \circ \phi_{\mathcal{P}}
  \circ \Phi_{\mathcal{F}\rightarrow \mathcal{P}}(p)
\end{equation}
where $\Phi_{\mathcal{F}\rightarrow \mathcal{P}}$ and $\phi_{\mathcal{P}}$ are defined above,
and~$\sigma(x)\triangleq \frac{e^x}{e^x+1}$ for~$x\in \mathbb{R}$
maps the unbounded range of~$\phi_{\mathcal{P}}$ to a finite interval~$[0,\,1]$.
The exact definition and detailed derivations for the transformation
in~\eqref{eq:complete-nf} are omitted here.
The readers are referred
to~\cite{rimon1992exact, loizou2017navigation, loizou2022mobile}
and our previous work~\cite{wang2025hybrid}.

\begin{definition}[Resulting Forest Path]\label{def:ass-path}
Given the harmonic potentials~$\varphi_\mathcal{F}(p)$ in $\mathcal{F}$,
its \emph{resulting forest path} from an initial point~$p_{\texttt{s}}\in \mathcal{W}_0$
is derived by following its negated gradient~$-\nabla_{p}\varphi_\mathcal{F}$,
i.e., ${\boldsymbol{\tau}}(t) \triangleq p_{\texttt{s}} - \int_{0}^{t} \nabla_p\varphi_\mathcal{F} ({\boldsymbol{\tau}}(\xi))\, d\xi$.
\hfill $\blacksquare$
\end{definition}
Similar definitions apply to the resulting path of the harmonic
potential~$\phi_{\mathcal{P}}(q)$ in~\eqref{eq:harmonic-point-potential},
denoted by~$\widetilde{\boldsymbol{\tau}}(t)$.
{Without loss of generality, the infinite time interval $t \in [0, +\infty)$
for both $\boldsymbol{\tau}(t)$ and $\widetilde{\boldsymbol{\tau}}(t)$
are mapped to $[0,\,1]$.}

\begin{figure}[t]
  \centering
`  \includegraphics[width=0.9\hsize]{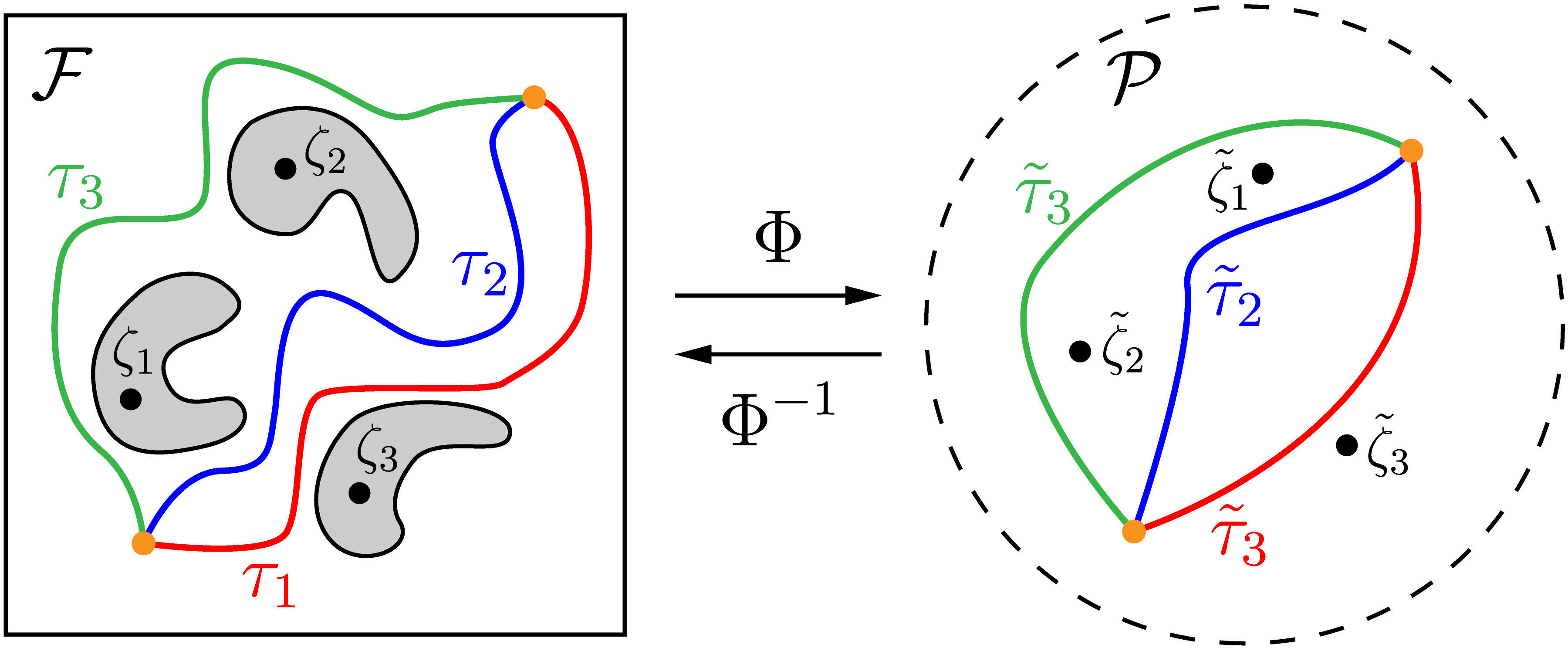}
  \vspace{-0.1in}
  \caption{\textbf{Left:} $\boldsymbol{\tau}_1$ and $\boldsymbol{\tau}_2$ belong to the same homotopy class,
  but $\boldsymbol{\tau}_3$ belongs to a different homotopy in forest world~$\mathcal{F}$;
  \textbf{Right:} homotopy classes are kept in point world~$\mathcal{P}$
  after the diffeomorphic transformation~$\Phi_{\mathcal{F}\rightarrow \mathcal{P}}$.
  }\label{fig:diffeo_homoto}
  \vspace{-0.15in}
\end{figure}

\subsection{Homotopy and Homology Classes of Paths}\label{subsec:homotopic}
Homotopy and homology classes are common tools
to distinguish different paths within the same workspace,
see~\cite{bhattacharya2012topological, bhattacharya2015persistent}.
Compared to homotopy, homology offers a computationally more tractable alternative.

\begin{definition}[Homotopic Paths]\label{def:homotopy-path}
Given two paths~$\boldsymbol{\tau}_1,\boldsymbol{\tau}_2:[0,\,1] \rightarrow \mathbb{R}^2$
that $\boldsymbol{\tau}_1(0)=\boldsymbol{\tau}_2(0)$ and $\boldsymbol{\tau}_1(1)=\boldsymbol{\tau}_2(1)$,
then~$\boldsymbol{\tau}_1$ and $\boldsymbol{\tau}_2$ are \emph{homotopic}
if there is a continuous map $\chi:[0,\, 1]\times[0,\, 1]\rightarrow\mathbb{R}^2$
that $\chi(\lambda,\, 0)=\boldsymbol{\tau}_1(\lambda)$,
$\forall \lambda \in [0,\, 1]$ and $\chi(\mu,\,1)=\boldsymbol{\tau}_2(\mu)$,
$\forall \mu \in [0,\, 1]$.
\hfill $\blacksquare$
\end{definition}
\begin{definition}[Homologous Paths]\label{def:homology-path}
Given two paths~$\boldsymbol{\tau}_1,\boldsymbol{\tau}_2:[0,\,1] \rightarrow \mathbb{R}^2$
that $\boldsymbol{\tau}_1(0)=\boldsymbol{\tau}_2(0)$ and $\boldsymbol{\tau}_1(1)=\boldsymbol{\tau}_2(1)$,
then $\boldsymbol{\tau}_1$ and $\boldsymbol{\tau}_2$ are \emph{homologous}
if there are no obstacles within the closed
loop~$L_{12} \triangleq(\boldsymbol{\tau}_1 \sqcup -\boldsymbol{\tau}_2)$,
i.e., $\texttt{Int}(L_{12}) \bigcap \mathcal{O}_i =\emptyset$, $\forall i \in \mathcal{M}$,
where~$\texttt{Int}(\cdot)$ is the interior of a closed-loop.
 \hfill $\blacksquare$
\end{definition}
Homotopy indicates that two paths $\boldsymbol{\tau}_1$ and $\boldsymbol{\tau}_2$
can be continuously deformed into the other without intersecting any obstacle,
whereas homology implies that the closed loop formed by $\boldsymbol{\tau}_1$ and $\boldsymbol{\tau}_2$
does not enclose any obstacle.
It has been proven in~\cite{bhattacharya2012topological} that homotopic paths
are also homologous since there always exists a homomorphism map from
the homotopy groups to the homology groups.
Besides, various  measures have been proposed
to classify paths into homotopy classes, e.g.,
H-signature~\cite{bhattacharya2012topological} and
visibility sequence~\cite{schmitzberger2002capture}.

\section{Problem Description}\label{sec:problem}
Consider a forest world~$\mathcal{F}$ as defined earlier
within a connected and compact workspace~$\mathcal{W}\subset \mathbb{R}^2$.
It contains~$M$ non-overlapping internal obstacles~$\mathcal{O}_i\subset {\mathcal{W}}_0$,
$\forall i\in \mathcal{M}$.
Specifically, each obstacle $\mathcal{O}_i$ has a tree-of-stars structure
composed of several star-shaped obstacles~\cite{rimon1992exact,
  loizou2017navigation, loizou2022mobile}.
A particular type of star-shaped obstacle called \emph{squircle}
is considered in this work.
As shown in Fig.~\ref{fig:hybrid}, a squircle interpolates smoothly between a circle and a square,
while avoiding non-differentiable corners~\cite{li2018navigation}, i.e.,
$O(p)\triangleq
\{ p \in \mathbb{R}^2\,|\, \beta(p)=\big(\|p\|^2+\sqrt{\|p\|^4-4\kappa^2\,
[(p^\intercal e_1)(p^\intercal e_2)]^2}\big)/2-1 \leq 0 \},$
where~$\kappa\in (0,\,1)$ is a positive parameter,
$\|\cdot\|$ is the Euclidean norm,
and $e_1, e_2$ are two base vectors in~$\mathbb{R}^2$.
Non-unit squircles can be derived by scaling, translation and rotation,
i.e., $O_{n_i}\triangleq \texttt{SQ}\big(c_{n_i},\,
w_{n_i},\,h_{n_i},\, \theta_{n_i}\big)$,
where~$c_{n_i}\in \mathbb{R}^{2}$ is the center;
$w_{n_i},\, h_{n_i}>0$ are the width and height;
and $\theta_{n_i}\in (-\pi,\, \pi]$ is the orientation,
$\forall n_i \in \{1,2, \cdots, N_i\} \triangleq \mathcal{N}_i$,
where~$N_i$ is the depth of the tree in~$\mathcal{O}_i$.

Given the start and goal~$p_{\texttt{s}},p_{\texttt{g}} \in \mathcal{F}$,
the objective is \emph{threefold}:
(I) construct valid harmonic potentials~$\varphi_{\mathcal{F}}(p)$ as
defined in~\eqref{eq:complete-nf};
(II) find potentials that the resulting path
belongs to a desired homotopy class;
and (III) determine all homotopy classes of paths that result from
valid potentials.


\section{Proposed Solution}\label{sec:solution}
The proposed solution consists of two main components:
(I) a simplified representation for homotopic paths in the point world,
via diffeomorphic transformation;
(II) a hybrid optimization algorithm to customize the harmonic potentials
in the point world, given a desired homotopy class.
\subsection{Homotopy Classes in Forest World}\label{subsec:Consistency}
\subsubsection{Invariance of Homology under Diffeomorphism}\label{subsubsec:consistency}
Via the diffeomorphic transformation, the forest world $\mathcal{F}$ can be mapped into a point world $\mathcal{P}$,
which has a much simpler geometry.
Thus, the complexity of finding homotopic paths can be reduced in  $\mathcal{P}$,
as shown in Fig.~\ref{fig:diffeo_homoto}.
For correctness, it is proven formally below that the homology properties of a
path remain unchanged after the transformation.
{Without loss of generality, both $\mathbb{R}^2$ and $\mathbb{C}$ represent the 2D plane interchangeably.}
Then, the obstacle marker function and H-signature function are defined
following~\cite{bhattacharya2012topological}:
\begin{definition}[Obstacle Marker]\label{def:obstacle-marker-function}
Given $M$ obstacles $\mathcal{O}_1, \mathcal{O}_2, \cdots, \mathcal{O}_M$ in the 2D complex plane $\mathbb{C}$,
the \emph{obstacle marker function} $\Gamma(z):\mathbb{C}\rightarrow \mathbb{C}^M$ is defined as:
\begin{equation}\label{eq:obstacle-marker-function}
  \Gamma(z) \triangleq
  [\frac{\gamma_1(z)}{z-\zeta_1}, \cdots, \frac{\gamma_M(z)}{z-\zeta_M}]^\intercal,
\end{equation}
where~$\gamma_\ell:\mathbb{C}\rightarrow \mathbb{C}$ is analytic;
$\zeta_\ell\in \mathcal{O}_\ell $ is any marker point within each obstacle $\mathcal{O}_\ell$;
and~$\gamma_\ell(\zeta_\ell) \neq 0$, $\forall \ell\in \mathcal{M}$. \hfill $\blacksquare$
\end{definition}
\begin{definition}[H-signature]\label{def:h-signature}
Given $M$ obstacles in the 2D complex plane $\mathbb{C}$,
the \emph{H-signature} of the path~$\boldsymbol{\tau} \in \mathcal{C}$
is defined as the vector function~$\mathcal{H}_2:
\mathcal{C} \rightarrow \mathbb{C}^M$ that:
\begin{equation}\label{eq:h-signature}
  \mathcal{H}_2(\boldsymbol{\tau}) \triangleq \int_{\boldsymbol{\tau}} \Gamma(z) dz,
\end{equation}
where~$\mathcal{C}$ is the set of all paths in $\mathbb{C}$;
$\Gamma(z)$ is the obstacle marker function from Def.~\ref{def:obstacle-marker-function}.
\hfill $\blacksquare$
\end{definition}
Given a smooth path~$\boldsymbol{\tau}$ within the forest world~$\mathcal{F}$,
the \emph{associated} path in point world~$\mathcal{P}$
{after} the transformation~$\Phi_{\mathcal{F}\rightarrow \mathcal{P}}$ from the forest world to the point world is given by:
$\widetilde{\boldsymbol{\tau}} \triangleq
\Phi(\boldsymbol{\tau}) =
\big\{\Phi_{\mathcal{F}\rightarrow \mathcal{P}}(p)\,|\, p \in \boldsymbol{\tau}\big\}$,
which is also smooth.
The details of the transformation $\Phi(\cdot)$ is given in Sec.~\ref{subsec:hybrid}.
The following lemma shows that the homology is invariant under
diffeomorphic transformations.

\begin{lemma}\label{lemma:homology-consistency}
Consider two paths $\boldsymbol{\tau}_1, \boldsymbol{\tau}_2 \subset \mathcal{F}$ that
$\boldsymbol{\tau}_1(0)=\boldsymbol{\tau}_2(0)$ and $\boldsymbol{\tau}_1(1)=\boldsymbol{\tau}_2(1)$,
of which the associated paths after the diffeomorphic transformation~$\Phi_{\mathcal{F}\rightarrow \mathcal{P}}$
in the point world $\mathcal{P}$
are $\widetilde{\boldsymbol{\tau}}_1$ and $\widetilde{\boldsymbol{\tau}}_2$.
Then, $\boldsymbol{\tau}_1$ and $\boldsymbol{\tau}_2$ are homologous
if and only if~$\widetilde{\boldsymbol{\tau}}_1$ and $\widetilde{\boldsymbol{\tau}}_2$ are homologous.
\end{lemma}
\begin{proof}
By definition, the marker points~$\zeta_\ell \in \mathcal{F}$
and~$\widetilde{\zeta}_\ell \in \mathcal{P}$ are related by
$\widetilde{\zeta}_\ell = \Phi(\zeta_\ell)$, $\forall \ell \in \mathcal{M}$.
Assume that~$\boldsymbol{\tau}_1$ and $\boldsymbol{\tau}_2$ are homologous paths,
but $\widetilde{\boldsymbol{\tau}}_1$ and $\widetilde{\boldsymbol{\tau}}_2$ are not homologous.
According to Lemma~2 in~\cite{bhattacharya2012topological},
the following holds:
\begin{equation}\label{eq:h-homologous}
  \mathcal{H}_2(\widetilde{\boldsymbol{\tau}}_1) - \mathcal{H}_2(\widetilde{\boldsymbol{\tau}}_2)
  = \int_{\widetilde{L}_{12}}
  \Gamma(\widetilde{z})\, d\widetilde{z} \neq \mathbf{0},
\end{equation}
where~$\widetilde{L}_{12}\triangleq (\widetilde{\boldsymbol{\tau}}_1 \sqcup -\widetilde{\boldsymbol{\tau}}_2)$;
and~$\Gamma(\widetilde{z})$ is defined in~\eqref{eq:obstacle-marker-function}.
Then,
since~$\widetilde{\boldsymbol{\tau}}_1 = \Phi(\boldsymbol{\tau}_1)$ and $\widetilde{\boldsymbol{\tau}}_2 = \Phi(\boldsymbol{\tau}_2)$ holds,
it follows that:
\begin{equation*}
\begin{aligned}
  & \int_{\widetilde{L}_{12}} \Gamma(\widetilde{z}) \, d\widetilde{z}
    = \int_{\widetilde{L}_{12}} \big[\frac{\gamma_1(\widetilde{z})}{\widetilde{z}
    -\widetilde{\zeta}_1},\cdots, \frac{\gamma_M(\widetilde{z})}{\widetilde{z}
    -\widetilde{\zeta}_M}\big]^\intercal \, d\widetilde{z}\\
  &= \int_{L_{12}} \big[\frac{\gamma_1(\Phi(z)) \nabla_z\Phi(z)}{\Phi(z)-\Phi(\zeta_1)},
    \cdots, \frac{\gamma_M(\Phi(z))\nabla_z\Phi(z)}{\Phi(z)-\Phi(\zeta_M)}\big]^\intercal dz,
\end{aligned}
\end{equation*}
where~$L_{12}\triangleq ({\boldsymbol{\tau}}_1 \sqcup - {\boldsymbol{\tau}}_2)$,
and~$\nabla_z\Phi(z)$ is the Jacobian.
Since~$\Phi(z)$ is analytic, $\nabla \Phi(z)$ is non-singular
and $\zeta_\ell \notin \texttt{Int}(L_{12})$, $\forall \ell \in \mathcal{M}$.
It in turn indicates that~$\frac{\gamma_\ell(\Phi(z)) \nabla\Phi(z)}{\Phi(z)-\Phi(\zeta_\ell)}$
is analytic, for all $z \in \texttt{Int}(L_{12})$ and $\ell \in\mathcal{M}$.
By the Cauchy Integral Theorem~\cite{zwillinger2021handbook},
the above integral equals to~$\mathbf{0}$,
which contradicts the assumption in~\eqref{eq:h-homologous}.
Therefore, $\widetilde{\boldsymbol{\tau}}_1$ and $\widetilde{\boldsymbol{\tau}}_2$ are homologous.
Given that~$\Phi$ is diffeomorphic,
the sufficiency can be proven similarly.
\end{proof}

\subsubsection{D-Signature}\label{subsubsec:metric}
The H-signature of a given path in the forest world
by~\eqref{eq:h-signature} requires integration over the entire path,
which is computationally expensive for general harmonic potentials.
Thus, a multi-directional distance signature is introduced to
distinguish homology classes more effectively in the point world.
Consider a point world~$\mathcal{P}$ with~$M$ point obstacles~$P_1, P_2, \cdots, P_M$.
The path~$\widetilde{\boldsymbol{\tau}}$ with $\widetilde{\boldsymbol{\tau}}(0)
= q_{\texttt{s}}$ and $\widetilde{\boldsymbol{\tau}}(1) = q_{\texttt{g}}$
has a D-signature defined below.
\begin{definition}[D-Signature]\label{def:multi-directional-distance}
The \emph{D-signature} for path~$\widetilde{\boldsymbol{\tau}}$
within the point world~$\mathcal{P}$ is given by:
\begin{equation}\label{eq:distance-metric}
  D(\widetilde{\boldsymbol{\tau}}) \triangleq S(\widetilde{\boldsymbol{\tau}})\, \odot\, \overline{D}(\widetilde{\boldsymbol{\tau}}),
\end{equation}
where the signed distance~$D(\widetilde{\boldsymbol{\tau}})\in \mathbb{R}^M$;
the sign~$S(\widetilde{\boldsymbol{\tau}}) \in \{+1,\, -1\}^M$;
the distance~$\overline{D}(\widetilde{\boldsymbol{\tau}})\in \mathbb{R}_{+}^M$;
and the symbol $\odot$ denotes the Hadamard product~\cite{boyd2004convex}.
\hfill $\blacksquare$
\end{definition}

As illustrated in Fig.~\ref{fig:multi-distance},
the D-signature~$D(\widetilde{\boldsymbol{\tau}})$ can
be computed in three steps:
(I) A circular area that covers all the point obstacles,
as well as the start and goal points,
is computed, denoted by~$C(\widetilde{\boldsymbol{\tau}})$,
centered at~$Q_{\texttt{c}}\triangleq (q_{\texttt{s}} + q_{\texttt{g}})/2$
and a radius~$R_{\texttt{c}}$.
Namely, it holds that~$\widetilde{\boldsymbol{\tau}} \subset \texttt{Int}(C)$
and~$P_i\in \texttt{Int}(C)$, $\forall i\in \mathcal{M}$;
(II) Denote by $q_{\texttt{g}}^{\prime}$ the intersections of the ray from $q_{\texttt{s}}$ to $q_{\texttt{g}}$
on the boundary~$\partial C$,
similarly~$q_{\texttt{s}}^{\prime}$ for~$q_{\texttt{s}}$.
Thus, the interior~$\texttt{Int}(C)$ is divided into two
parts:
\begin{equation}\label{eq:semi-areas}
\begin{aligned}
  C^{+}(\widetilde{\boldsymbol{\tau}}) &\triangleq
\texttt{Int}\Big(\widetilde{\boldsymbol{\tau}} \sqcup \, \overrightarrow{q_{\texttt{g}} q_{\texttt{g}}^{\prime}}
                \sqcup \, \wideparen{q_{\texttt{g}}^{\prime} q_{\texttt{s}}^{\prime}}
                \sqcup \, \overrightarrow{q_{\texttt{s}} q_{\texttt{s}}^{\prime}}\Big); \\
                C^{-}(\widetilde{\boldsymbol{\tau}}) &\triangleq C(\widetilde{\boldsymbol{\tau}})
                \backslash C^{+}(\widetilde{\boldsymbol{\tau}}),
\end{aligned}
\end{equation}
where $C^{+}(\widetilde{\boldsymbol{\tau}})$ is the area enclosed
by the \emph{anticlockwise} loop formed with the path~$\widetilde{\boldsymbol{\tau}}$
from $q_{\texttt{s}}$ to $q_{\texttt{g}}$,
the line~$\overrightarrow{q_{\texttt{g}} q_{\texttt{g}}^{\prime}}$ from $q_{\texttt{g}}$ to $q_{\texttt{g}}^{\prime}$,
the arc $\wideparen{q_{\texttt{g}}^{\prime} q_{\texttt{s}}^{\prime}}$ from $q_{\texttt{g}}^{\prime}$
to $q_{\texttt{s}}^{\prime}$,
and the line $\overrightarrow{q_{\texttt{s}} q_{\texttt{s}}^{\prime}}$ from $q_{\texttt{s}}$ to $q_{\texttt{s}}^{\prime}$;
and $C^{-}(\widetilde{\boldsymbol{\tau}})$ is the complement
of $C^{+}(\widetilde{\boldsymbol{\tau}})$ within~$C$; (III) The sign~$S(\widetilde{\boldsymbol{\tau}})\triangleq [s_i]$ is given by~$s_i=1$
if~$P_i \in C^{+}(\widetilde{\boldsymbol{\tau}})$,
while~$s_i=-1$ if~$P_i \in C^{-}(\widetilde{\boldsymbol{\tau}})$, $\forall i \in\mathcal{M}$.
Moreover, the distance~$\overline{D}(\widetilde{\boldsymbol{\tau}})\triangleq [\overline{d}_i]$ is given
by~$\overline{d}_i = \textbf{min}_{q \in \widetilde{\boldsymbol{\tau}}} \|q - P_i\|$, $\forall i \in\mathcal{M}$.

\begin{figure}[t]
  \centering
  \includegraphics[width=0.85\hsize]{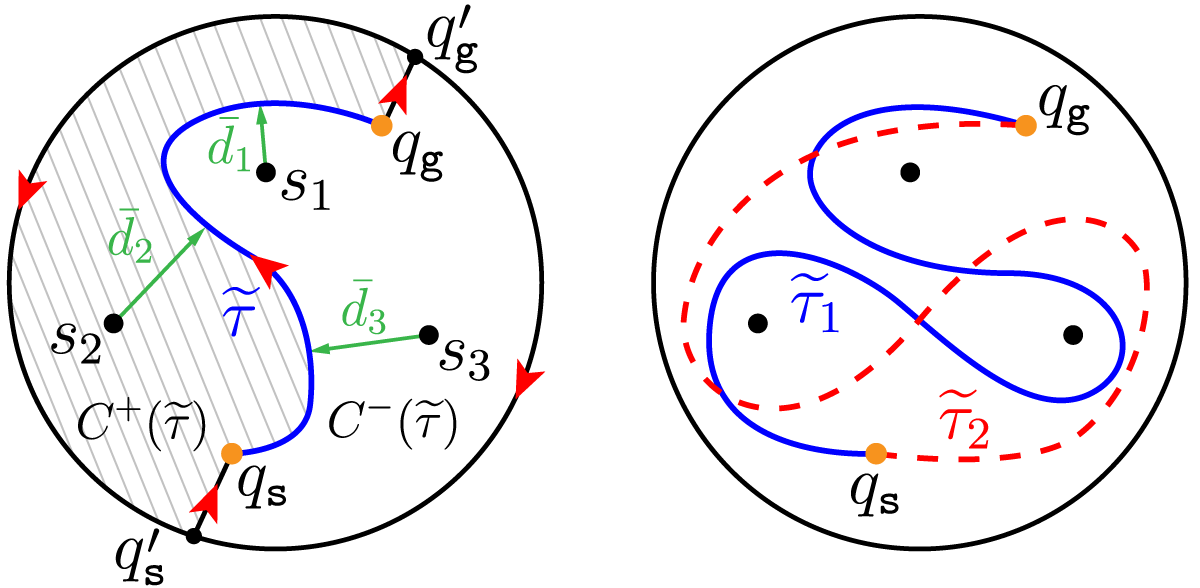}
  \vspace{-0.1in}
  \caption{\textbf{Left:} computation of the D-signature $D(\widetilde{\boldsymbol{\tau}})$ in~\eqref{eq:distance-metric};
  \textbf{Right:} $\widetilde{\boldsymbol{\tau}}_1$ and $\widetilde{\boldsymbol{\tau}}_2$ are homologous, but not homotopic.
  }\label{fig:multi-distance}
  \vspace{-0.17in}
\end{figure}

\begin{lemma}\label{lemma:homology-metric}
Consider two paths $\widetilde{\boldsymbol{\tau}}_1, \widetilde{\boldsymbol{\tau}}_2 \subset \mathcal{P}$ that
$\widetilde{\boldsymbol{\tau}}_1(0) = \widetilde{\boldsymbol{\tau}}_2(0) = q_{\texttt{s}}$
and $\widetilde{\boldsymbol{\tau}}_1(1) = \widetilde{\boldsymbol{\tau}}_2(1) = q_{\texttt{g}}$,
then $\widetilde{\boldsymbol{\tau}}_1$ and $\widetilde{\boldsymbol{\tau}}_2$
are homologous if and only if $S(\widetilde{\boldsymbol{\tau}}_1)=S(\widetilde{\boldsymbol{\tau}}_2)$.
\end{lemma}
\begin{proof}
(\emph{Sufficiency}): The proof is by contradiction.
First, assume that~$\widetilde{\boldsymbol{\tau}}_1$ and $\widetilde{\boldsymbol{\tau}}_2$
are homologous but $S(\widetilde{\boldsymbol{\tau}}_1) \neq S(\widetilde{\boldsymbol{\tau}}_2)$.
Due to Def.~\ref{def:multi-directional-distance},
there exits~$K \geq 1$ point obstacles $P_{m_1}, P_{m_1}, \cdots, P_{m_K}$ such that
$P_{m_\ell} \in C^{+}(\widetilde{\boldsymbol{\tau}}_1)$ and $P_{m_\ell} \notin C^{+}(\widetilde{\boldsymbol{\tau}}_2)$,
$\forall \ell\in \{1, 2,\cdots,K\}\triangleq \mathcal{K}\subseteq \mathcal{M}$.
Furthermore, it follows that:
$
C^{+}(\widetilde{\boldsymbol{\tau}}_1) \backslash C^{+}(\widetilde{\boldsymbol{\tau}}_2)
  = \texttt{Int}\Big(\widetilde{\boldsymbol{\tau}}_1 \sqcup \overrightarrow{q_{\texttt{g}} q_{\texttt{g}}^{\prime}} \sqcup
      \wideparen{q_{\texttt{g}}^{\prime} q_{\texttt{s}}^{\prime}} \sqcup \overrightarrow{q_{\texttt{s}} q_{\texttt{s}}^{\prime}}) \, \backslash
      \texttt{Int}\Big(\widetilde{\boldsymbol{\tau}}_2 \sqcup \overrightarrow{q_{\texttt{g}} q_{\texttt{g}}^{\prime}}
      \sqcup \wideparen{q_{\texttt{g}}^{\prime} q_{\texttt{s}}^{\prime}} \sqcup \overrightarrow{q_{\texttt{s}} q_{\texttt{s}}^{\prime}}\Big)
= \texttt{Int}\big(\widetilde{\boldsymbol{\tau}}_1 \sqcup -\widetilde{\boldsymbol{\tau}}_2\big) = \widetilde{L}_{12},
$
which implies that $P_{m_\ell} \in \texttt{Int}(\widetilde{L}_{12})$,
$\forall \ell\in  \mathcal{K}$.
Via Def.~\ref{eq:h-signature} and the Residue Theorem~\cite{zwillinger2021handbook},
it holds that:
\begin{equation}
\begin{aligned}
  & \mathcal{H}_2(\widetilde{\boldsymbol{\tau}}_1) - \mathcal{H}_2(\widetilde{\boldsymbol{\tau}}_2)
  = \int_{\widetilde{\boldsymbol{\tau}}_1} \Gamma(\widetilde{z}) d\widetilde{z} - \int_{\widetilde{\boldsymbol{\tau}}_2} \Gamma(\widetilde{z}) d\widetilde{z}\\
= & 2\pi j \sum_{\ell=1}^K \underset{\widetilde{z} \to P_{m_\ell}}{\textbf{lim}}(\widetilde{z} - P_{m_\ell})
[\frac{\gamma_1(\widetilde{z})}{\widetilde{z}-P_1}, \cdots \frac{\gamma_M(\widetilde{z})}{\widetilde{z}-P_M}]^\intercal\\
= & [\cdots,\, \gamma_{m_1}(P_{m_1}),\, \cdots,\, \gamma_{m_K}(P_{m_K}),\,
  \cdots]^\intercal \neq \mathbf{0},
\end{aligned}
\end{equation}
which contradicts the assumption that~$\widetilde{\boldsymbol{\tau}}_1$ and $\widetilde{\boldsymbol{\tau}}_2$ are homologous.
Thus,$S(\widetilde{\boldsymbol{\tau}}_1) = S(\widetilde{\boldsymbol{\tau}}_2)$ and the sufficiency holds.

(\emph{Necessity}): Assume that~$S(\widetilde{\boldsymbol{\tau}}_1) = S(\widetilde{\boldsymbol{\tau}}_2)$ holds
and no obstacles exist in $C^{+}(\widetilde{\boldsymbol{\tau}}_1) \backslash C^{+}(\widetilde{\boldsymbol{\tau}}_2)=
\texttt{Int}(\widetilde{\boldsymbol{\tau}}_1 \sqcup -\widetilde{\boldsymbol{\tau}}_2)=\texttt{Int}(\widetilde{L}_{12})$.
By the Cauchy Integral Theorem~\cite{zwillinger2021handbook},
$\mathcal{H}_2(\widetilde{\boldsymbol{\tau}}_1) - \mathcal{H}_2(\widetilde{\boldsymbol{\tau}}_2)=\int_{\widetilde{L}_{12}} \Gamma(\widetilde{z}) d\widetilde{z}=\mathbf{0}$,
which concludes the proof.
\end{proof}

\begin{figure}[t]
  \centering
  \includegraphics[width=0.95\hsize]{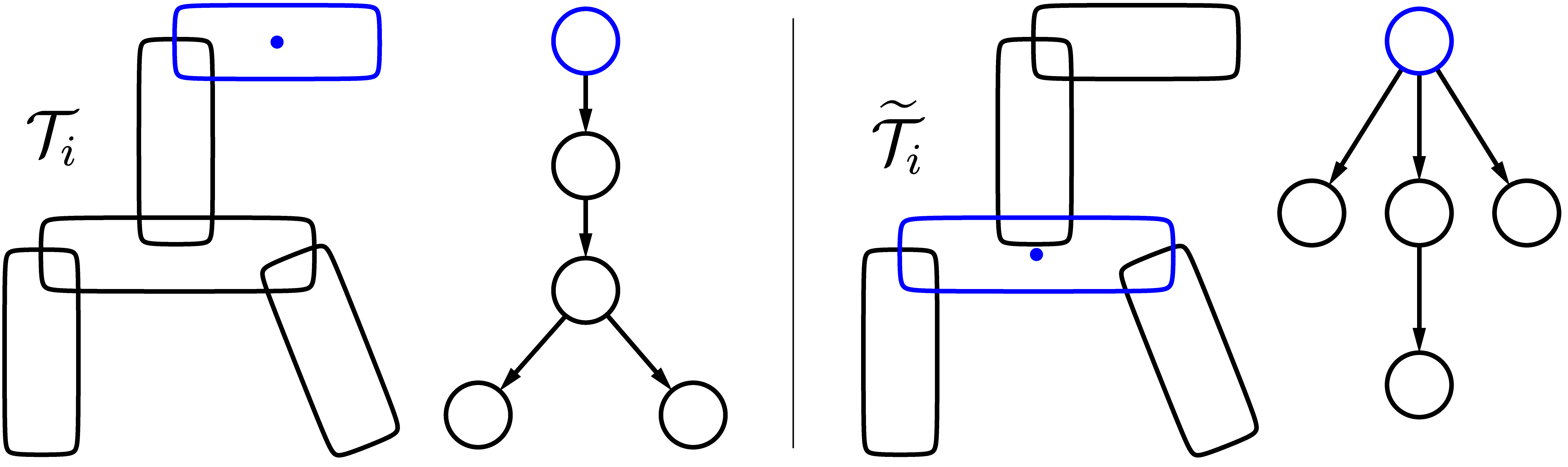}
  \vspace{-0.1in}
  \caption{Different tree structures representing
    the same overlapping squircles,
    where the squircle (in blue) is chosen as the root.
  }\label{fig:reroot}
  \vspace{-0.2in}
\end{figure}

\begin{lemma}\label{lemma:homotopy}
Consider two paths $\widetilde{\boldsymbol{\tau}}_1, \widetilde{\boldsymbol{\tau}}_2 \subset \mathcal{P}$ that
$\widetilde{\boldsymbol{\tau}}_1(0) = \widetilde{\boldsymbol{\tau}}_2(0) = q_{\texttt{s}}$
and $\widetilde{\boldsymbol{\tau}}_1(1) = \widetilde{\boldsymbol{\tau}}_2(1) = q_{\texttt{g}}$.
If $S(\widetilde{\boldsymbol{\tau}}_1) \neq S(\widetilde{\boldsymbol{\tau}}_2)$ holds,
then $\widetilde{\boldsymbol{\tau}}_1$ and $\widetilde{\boldsymbol{\tau}}_2$ belong to different homotopy classes.
\end{lemma}
\begin{proof}
Assume that~$S(\widetilde{\boldsymbol{\tau}}_1) \neq S(\widetilde{\boldsymbol{\tau}}_2)$
but $\widetilde{\boldsymbol{\tau}}_1$ and $\widetilde{\boldsymbol{\tau}}_2$ are homotopic.
It has been shown~\cite{bhattacharya2012topological} that
homotopic paths are homologous,
thus $\widetilde{\boldsymbol{\tau}}_1$ and $\widetilde{\boldsymbol{\tau}}_2$ are homologous.
By Lemma~\ref{lemma:homology-metric}, $S(\widetilde{\boldsymbol{\tau}}_1) = S(\widetilde{\boldsymbol{\tau}}_2)$,
which contradicts the assumption.
\end{proof}

\begin{remark}\label{remark:obstacle-shape}
Lemma~\ref{lemma:homology-metric}
states that the sign function~$S(\widetilde{\boldsymbol{\tau}})$ is equivalence to a \emph{homology} class,
while Lemma~\ref{lemma:homotopy} implies that
each distinct $S(\widetilde{\boldsymbol{\tau}})$ corresponds to a different \emph{homotopy} class.
However, the necessity of the condition in Lemma~\ref{lemma:homotopy}
does not hold, as some different homotopy classes share the same
sign function $S(\widetilde{\boldsymbol{\tau}})$,
as shown in Fig~\ref{fig:multi-distance}.
  \hfill $\blacksquare$
\end{remark}

\subsubsection{Synthesis of Homotopic Paths}\label{subsubsec:search}
To find paths of different homotopy classes within~$\mathcal{P}$,
it is essential to first generate reference paths of different classes
for optimization.
Since each sign $S(\cdot)$ represents one class,
the set of D-signatures associated with potential homotopy classes
is given by:
\begin{equation}\label{eq:all_homo}
  \begin{split}
    &\mathscr{D}^\star \triangleq \Big\{S^\star \odot \overline{D}^\star
    \, | \,
    S^\star \in \{+1, -1\}^M,\, \overline{D}^\star=[d^\star_i(S^\star)]\Big\};\\
  &\overline{d}_i^\star(S^\star) = \underset{j\in \mathcal{M}, s_i^\star + s_j^\star=0}{\textbf{min}}
  \Big{\{}\|P_i - P_j\|\Big{\}}, \; \forall i\in\mathcal{M};
\end{split}
\end{equation}
where $S^\star \triangleq [s_i^\star]$ and
$\overline{D}^\star \triangleq [\overline{d}_i^\star]$.
The second condition implies that the closest obstacle~$P_j$
with the opposite sign is found for each obstacle~$P_i$, $\forall i, j \in \mathcal{M}$,
between which the relative distance is used as the distance measure
for the path.
Consequently, any D-signature~$D^\star\in \mathscr{D}^\star$
is associated with a homotopy class,
which is the objective for the hybrid optimization scheme in the sequel.

\subsection{Hybrid Optimization of Harmonic Potentials}\label{subsec:hybrid}
Given a homotopy class associated with the D-signature~$D^\star$,
it remains a question whether there exists a harmonic potential~$\varphi_{\mathcal{F}}(p)$,
of which the resulting path from the same initial point belongs to the same homotopy class.
\subsubsection{Structure Selection for Tree-of-Stars}
\label{subsubsec:forest_world}
As introduced in~\cite{rimon1992exact, wang2025hybrid},
a forest world consists of several disjointed groups of obstacles as trees-of-stars.
As shown in Fig.~\ref{fig:reroot}, each group is a finite union of overlapping obstacles whose adjacency graph is a tree
with a unique root.
The choice of these tree structures determines the locations of the obstacles in the point world,
which in turn changes the underlying potential fields.
Consequently, the tree structure can be customized based on the desired homotopy classes.

\begin{definition}[Tree-of-Squircles]\label{def:tree-of-squircles}
A tree-of-squircles is a directed acyclic tree
$\mathcal{T} \triangleq (\mathcal{O},\, E,\, O^\star)$,
where: (I)~$\mathcal{O}$ is the set of obstacles within~$\mathcal{T}$;
(II)~$E\subset \mathcal{O} \times \mathcal{O}$ is the set of directed parent-child relations,
i.e., $(O'_{\ell},\, O_{\ell}) \in E$ if squircles $O'_{\ell}$ and $O_{\ell}$ overlap;
and (III)~$O^\star\in \mathcal{O}$ is the root.\hfill $\blacksquare$
\end{definition}

A tree~$\mathcal{T}_i\triangleq (\mathcal{O}_i,\, E_i,\, O_i^\star)$
is \emph{valid} if each squircle has exactly one parent squircle, except for the root.
Moreover, a valid tree~$\mathcal{T}_i$ can be \emph{re-rooted} by choosing a
different root,
and the rest of the nodes are re-structured according to their adjacency relation
via a breadth-first search,
yielding a new tree~$\widetilde{\mathcal{T}}_i$, as shown in Fig.~\ref{fig:reroot}.
The lemma below ensures that the re-rooted tree is unique and remains valid.


\begin{lemma}\label{lemma:valid}
Given a valid tree~$\mathcal{T}_i=(\mathcal{O}_i,\, E_i,\, O_i^\star)$,
the re-rooted tree~$\widetilde{\mathcal{T}}_i=(\mathcal{O}_i,\,\widetilde{E}_i,\,\widetilde{O}^\star_{i})$
is unique and valid for any chosen new root~$\widetilde{O}^\star_{i} \in \mathcal{O}_i$.
\end{lemma}
\begin{proof}
There is a unique path from the original root~$O_i^\star$ to
any other node~$O_\ell \in \mathcal{O}_i$ within~$\mathcal{T}_i$.
Given the new root~$\widetilde{O}^\star_{i}$, denote
by~$\widehat{\mathcal{O}}^\star_{i}\triangleq \{O\in \mathcal{O}_i\,|\,
(\widetilde{O}^\star_{i}, O) \in E_i \lor
(O, \widetilde{O}^\star_{i}) \in E_i\}$
its set of parents and children of~$\widetilde{O}^\star_{i}$ within~$\mathcal{T}_i$.
Since each squircle~$O_\ell \in \widehat{\mathcal{O}}^\star_{i}$
shares a common center with~$\widetilde{O}^\star_{i}$,
$O_\ell$ is re-assigned as the children of~$\widetilde{O}^\star_{i}$
in~$\widetilde{\mathcal{T}}_i$.
Thus, it holds that~$(\widetilde{O}^\star_{i}, O_\ell) \in \widetilde{E}_i$ within~$\widetilde{\mathcal{T}}_i$.
This process is repeated until all nodes in $\mathcal{O}_i$
are re-assigned in~$\widetilde{\mathcal{T}}_i$.
Since~$\mathcal{T}_i$ is acyclic and each node has one parent,
the re-rooted tree~$\widetilde{\mathcal{T}}_i$ is unique and remains valid.
\end{proof}

Given a forest structure~$\mathcal{F}$,
the transformation~$\Phi_{\mathcal{F}\rightarrow \mathcal{P}}$
from forest world to point world in~\eqref{eq:complete-nf}
is determined via several
diffeomorphic transformations~\cite{loizou2017navigation, wang2025hybrid}, i.e.,
\begin{equation}\label{eq:purge}
  \Phi_{\mathcal{F}\rightarrow \mathcal{P}}(p) \triangleq
  \psi \circ
  \Phi_{\mathcal{M}\rightarrow \overline{\mathcal{P}}} \circ
  \Phi_{\mathcal{S}\rightarrow \mathcal{M}} \circ
  \Phi_{\mathcal{F}\rightarrow \mathcal{S}}(p),
\end{equation}
where $\Phi_{\mathcal{F}\rightarrow \mathcal{S}}(p)$ is the \emph{purging transformation}:
$\Phi_{\mathcal{F}\rightarrow \mathcal{S}}(p) \triangleq \Phi_{M} \circ \cdots  \Phi_{2}\circ \Phi_{1}(p)$.
Here, $\Phi_{i}(p) \triangleq f_{i,1} \circ f_{i,2} \cdots \circ f_{i,N_i}(p)$ is the purging map for the~$i$-th
tree~$\mathcal{T}_i$, $\forall i \in \mathcal{M}$,
where~$f_{i,n_i}(\cdot)$ is the transformation for the~$n_i$-th leaf
with~$n_i \in \{1, 2,\cdots,N_i\}\triangleq \mathcal{N}_i$ and the depth~$N_i>0$,
i.e.,
$f_{i,n_i}(p) \triangleq p\,\big(1-\sigma_{n_i}(p)\big)+ \sigma_{n_i}(p)\, T_{n_i}(p)$,
where $T_{n_i}(p)$ is a scaling function
mapping the boundary of the child squircle to its parent;
and $\sigma_{n_i}(p)$ is the analytic switch
that attains value of one and vanishes nears other squircles and the goal.
Moreover,
$\Phi_{\mathcal{S}\rightarrow \mathcal{M}}$ and $\Phi_{\mathcal{M}\rightarrow \overline{\mathcal{P}}}$
are the transformations from star world to sphere world and sphere world to bounded point world.
The exact derivations are omitted here due to limited space.
Note that
$p_i^\star =
  \Phi_{\mathcal{M}\rightarrow \overline{\mathcal{P}}} \circ
  \Phi_{\mathcal{S}\rightarrow \mathcal{M}} \circ
  \Phi_{\mathcal{F}\rightarrow \mathcal{S}}(p_i^\star)$
holds for $p_i^\star$ being the center of the
root~$\widetilde{O}^\star_{i}$ for~$\mathcal{T}_i$.
Lastly, the function
$\psi(\widetilde{q}) \triangleq \frac{\rho_0}{\rho_0 - \|\widetilde{q} - q_0\|}(\widetilde{q} - q_0) + q_0$
transforms the bounded point world $\overline{\mathcal{P}}$
to the unbounded point world $\mathcal{P}$,
where $\rho_0 $ and $q_0$ are the radius and center of the sphere world.
Therefore, the point obstacles in~$\mathcal{P}$ are given
by~$\mathscr{P} \triangleq \{P_i\}$ with~$P_i = \psi(p_i^\star), \forall i \in \mathcal{M}$.

Note that there are~$M$ trees-of-squircles~$\{\mathcal{T}_1, \mathcal{T}_2, \cdots, \mathcal{T}_M\}$
within the forest world~$\mathcal{F}$.
Since the root of each tree can be re-selected as described,
yielding a large variety of possible forests, denoted by~$\mathscr{F}$.
For a given reference homotopy class~$D^\star \in \mathscr{D}^\star$,
the Fisher distance~\cite{boyd2004convex} is proposed below.
\begin{definition}[Fisher Distance]\label{def:fisher-discriminant}
Given the forest~$\mathcal{F}$
and the homotopy class~$D^\star$,
their \emph{Fisher distance} is defined as:
\begin{equation}\label{fisher-discriminant}
  J(D^\star,\, \mathcal{F}) \triangleq
  \frac{\|\overline{P}_1 - \overline{P}_2\|^2}{(\Sigma_1)^2 + (\Sigma_2)^2},
\end{equation}
where~$\mathscr{P}_1 \triangleq \{P_i \in \mathscr{P}| s_i=-1, i\in \mathcal{M}\}$,
$\mathscr{P}_2 \triangleq \{P_i \in  \mathscr{P}| s_i=1, i\in\mathcal{M}\}$;
$\overline{P}_\ell$ and $\Sigma_\ell$ are the mean and variance
of the point obstacles within~$\mathscr{P}_\ell$, respectively for~$\ell =1, 2$.
\hfill $\blacksquare$
\end{definition}
The Fisher distance measures the spatial distribution of point obstacles~$\{P_i\}$
relative to the given homotopy class~$D^\star$.
A higher Fisher distance indicates that the obstacles tend to be separated into two groups
on either side of the trajectory~$\widetilde{\boldsymbol{\tau}}$,
given the desired homotopy class.
Therefore, this separation renders the desired class
easier to find by the subsequent parameter optimization.
Consequently, the candidate forest structure can be selected:
\begin{equation}\label{eq:optimal-root}
  \mathcal{F}^\star \triangleq \textbf{max}_{\mathcal{F} \in \mathscr{F}} J(D^\star,\, \mathcal{F}),
\end{equation}
where~$\mathscr{F}$ is the set of all tree structures described earlier;
and~$\mathcal{F}^\star$ is the best structure for the class~$D^\star$.
Once the tree structure is selected, the diffeomorphic
transformation~$\Phi_{\mathcal{F}\rightarrow \mathcal{P}}$
from the forest world to the point world in~\eqref{eq:complete-nf}
is determined via the \emph{purging} process in~\eqref{eq:purge}.

\subsubsection{Parameter Optimization for Harmonic Potentials}\label{subsubsec:harmonic-potentials}
Once the forest structure~$\mathcal{F}^\star$ is selected,
the associated point world~$\mathcal{P}$ is derived.
Thus, it remains to design the weight parameters~$\mathbf{w}$ for the harmonic
potential~$\phi_{\mathcal{P}}(q)$ in~\eqref{eq:harmonic-point-potential},
such that the resulting path in~$\mathcal{P}$ has a D-signature close
to the desired~$D^\star$.
To begin with, the weights are constrained
for~$\phi_{\mathcal{P}}(q)$ to be valid.

\begin{figure}[t]
  \centering
  \includegraphics[width=0.9\hsize]{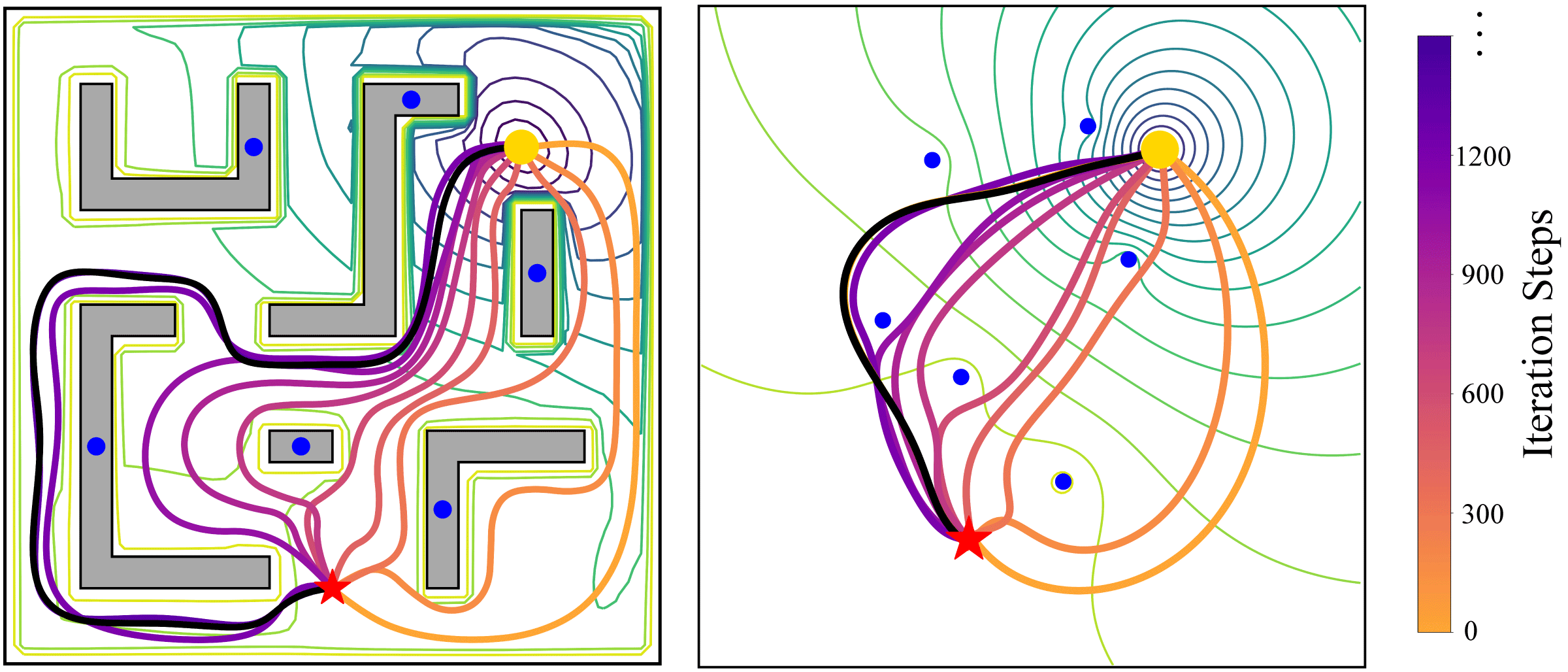}
  \vspace{-0.1in}
  \caption{Evolution of the resulting
    paths~$\widetilde{\boldsymbol{\tau}}(t, \mathbf{w})$ in~$\mathcal{F}$ and $\mathcal{P}$
    as the weights~$\mathbf{w}$
    are optimized by~\eqref{eq:opt-w} given the desired~$D^\star$ (in black).}
  \label{fig:gd_process}
  \vspace{-0.15in}
\end{figure}

\begin{lemma}\label{lemma:choice-of-weight}
If the weights~$\mathbf{w}$ in~$\phi_{\mathcal{P}}$ satisfy that
$w_{\texttt{g}}\geq\sum_{i\in \mathcal{M}} w_i + 1$,
its resulting path~$\widetilde{\boldsymbol{\tau}}$ in~$\mathcal{P}$
converges to the goal~$q_{\texttt{g}}$ without intersecting the
obstacles in~$\overline{\mathcal{P}}=\{P_i\}$,
from almost all initial point~$q_{\texttt{s}}\in \mathcal{P}$.
\end{lemma}
\begin{proof}
Given the above condition,
it follows from~\eqref{eq:harmonic-point-potential} that:
\begin{equation*}
\begin{aligned}
& \underset{\|q\| \to \infty}{\textbf{lim}}  \phi_{\mathcal{P}}(q)
  = \underset{\|q\| \to \infty}{\textbf{lim}} \ln \Big(\frac{\|q-q_{\texttt{g}}\|^{2\,w_{\texttt{g}}}}
  {\prod_{i=1}^{M} \|q-P_i\|^{2 w_i}} \Big) \\
  & = \underset{\|q\| \to \infty}{\textbf{lim}} \ln \Big(\|q\|^{2 \,(w_{\texttt{g}}-\sum_{i=1}^{M} w_i)} \Big)
  \geq \underset{\|q\| \to \infty}{\textbf{lim}} \ln \big(\|q\|^2 \big),
\end{aligned}
\end{equation*}
which reaches infinity.
Thus, the path~$\widetilde{\boldsymbol{\tau}}$
can not escape to infinity.
Moreover, the inner product between~$-\nabla_q \phi_{\mathcal{P}}(q)$
and $q-P_i$ as $q$ approaches $P_i\in \overline{\mathcal{P}}$ is given by:
\begin{equation*}
\begin{aligned}
  &\underset{q \to P_i}{\textbf{lim}} \big{\langle} -\nabla_q \phi_{\mathcal{P}}(q),\, q-P_i
    \big{\rangle}
    = \underset{q \to P_i}{\textbf{lim}} \Big{\langle} 2\mathbf{w}^\intercal\cdot
      \Big{[}\frac{-(q-q_{\texttt{g}})}{\|q-q_{\texttt{g}}\|^2},\\
      & \frac{(q-P_1)}{\|q-P_1\|^2}, \cdots, \frac{(q-P_M)}{\|q-P_M\|^2}
\Big{]}^\intercal,
q-P_i\Big{\rangle}
= \,\, 2 w_i > 0,
\end{aligned}
\end{equation*}
which indicates that the negated gradient directs outwards in the close vicinity of each
point obstacle~$P_i\in \overline{\mathcal{P}}$.
Thus, the path~$\widetilde{\boldsymbol{\tau}}$ is guaranteed to
avoid intersection with all point obstacles.
Lastly, Theorem~1 from~\cite{loizou2022mobile} has shown that
the potential~$\phi_{\mathcal{P}}(q)$ has a unique global minimum
at~$q_{\texttt{g}}$ and more importantly,
a set of isolated saddle points of measure zero due to its harmonic properties.
Therefore, the path~$\widetilde{\boldsymbol{\tau}}$ is collision-free
and converges to~$q_{\texttt{g}}$ asymptotically
under the aforementioned constraint.
This completes the proof.
\end{proof}

Next, the objective is to optimize the weights~$\mathbf{w}$ such that the D-signature
of the path~$\widetilde{\boldsymbol{\tau}}$ is close to the desired signature~$D^\star$,
i.e., as a constrained optimization problem:
\begin{equation}\label{eq:w-opt}
  \mathbf{w}^\star \triangleq \underset{\mathbf{w}\in \mathbf{W}}{\textbf{min}}\;
  \big{\|}D(\widetilde{\boldsymbol{\tau}}(t,\, \mathbf{w})) - D^\star\big{\|},
\end{equation}
where~$\mathbf{W}\triangleq \{\mathbf{w}>0|w_{\texttt{g}}>\sum_{i\in \mathcal{M}} w_i + 1\}$;
and $\mathbf{w}$ is added as an explicit parameter to the path~$\widetilde{\boldsymbol{\tau}}$.
Note that the point obstacles in~$\mathcal{P}$ are fixed given the forest structure~$\mathcal{F}^\star$.
Thus, the resulting path~$\widetilde{\boldsymbol{\tau}}(t, \mathbf{w})$ is fully determined
by the weights~$\mathbf{w}$.
Thus, a \emph{projected} gradient descent
method~\cite{boyd2004convex} is adopted:
\begin{subequations}\label{eq:opt-w}
\begin{align}
  &\mathbf{w}'_{k+1} = \mathbf{w}_{k} - \Big(D(\widetilde{\boldsymbol{\tau}})
  - D^\star\Big)^\intercal\frac{\partial D(\widetilde{\boldsymbol{\tau}}(t,
    \mathbf{w}_k))}{\partial \mathbf{w}_k}\, \delta_k; \label{eq:opt-gradient}\\
  &\mathbf{w}_{k+1} = \underset{\mathbf{w}\in \mathbf{W}_{\eta}}
{\textbf{argmin}}\; \|\mathbf{w} - \mathbf{w}'_{k+1}\|^2\label{eq:opt-proj},
\end{align}
\end{subequations}
where the second term in~\eqref{eq:opt-gradient} is the Jacobian
of~$D(\widetilde{\boldsymbol{\tau}}(t, \mathbf{w}))$ with respect to $\mathbf{w}$;
$\delta_k>0$ is the step size;
and the optimization in~\eqref{eq:opt-proj} is the
Euclidean projection that maps~$\mathbf{w}'_{k+1}$ back to the feasible
region~$\mathbf{W}$ with a certain margin~$\eta>0$.

Since the resulting path~$\widetilde{\boldsymbol{\tau}}(t, \mathbf{w})$
does not have an analytical solution,
its D-signature and the Jacobian $\frac{\partial D(\widetilde{\boldsymbol{\tau}}
(t, \mathbf{w}))}{\partial \mathbf{w}} \in \mathbb{R}^{M\times (M+1)}$
in~\eqref{eq:opt-gradient} is approximated numerically by:
\begin{equation*}
  \frac{\partial\, d_i(\widetilde{\boldsymbol{\tau}}(t, \mathbf{w}))}
       {\partial\, w_j} = \frac{d_i(\widetilde{\boldsymbol{\tau}}
         (t, \mathbf{w} + \boldsymbol{\epsilon}_j))
         - d_i(\widetilde{\boldsymbol{\tau}}
         (t, \mathbf{w} - \boldsymbol{\epsilon}_j))}{2\epsilon},
\end{equation*}
where~$\boldsymbol{\epsilon}_j \triangleq
[0, \cdots, \epsilon, \cdots, 0]^\intercal$ with $\epsilon>0$ at the $j$-th index,
for $i\in \mathcal{M}$ and $j\in \{0\} \cup \mathcal{M}$,
i.e., the first column is the partial derivative for $w_{\texttt{g}}$;
$\widetilde{\boldsymbol{\tau}}(t, \mathbf{w} + \boldsymbol{\epsilon}_j)$
is the new resulting path after the weights are changed to~$\mathbf{w} + \boldsymbol{\epsilon}_j$,
which should be computed numerically, i.e., via the Runge-Kutta methods~\cite{zwillinger2021handbook};
the same applies to~$\widetilde{\boldsymbol{\tau}}(t, \mathbf{w} - \boldsymbol{\epsilon}_j)$;
and their D-signatures~$D(\cdot)$ are computed by~\eqref{eq:distance-metric}.
Lastly, the solution of~\eqref{eq:opt-proj} can be derived analytically
through the Lagrange equation and the KKT conditions~\cite{boyd2004convex},
of which the details are omitted here due to limited space.
The iteration terminates after the Jacobian is less than a given threshold
or after certain number of iterations.
As shown in Fig.~\ref{fig:gd_process},
the resulting path~$\widetilde{\boldsymbol{\tau}}(t, \mathbf{w})$
gradually converges to the desired homotopy class.

\begin{remark}\label{remark:bifurcation}
  Although the resulting path~$\widetilde{\boldsymbol{\tau}}(t, \mathbf{w})$
  is continuous w.r.t. the parameter~$\mathbf{w}$,
  the discrete change of homotopy class is related to the phenomenon of
  \emph{bifurcation} in nonlinear differential equations~\cite{zwillinger2021handbook},
  when the path intersects with the neighborhood of the saddle
  points of the potential~$\phi_{\mathcal{P}}$.
  \hfill $\blacksquare$
\end{remark}

\begin{algorithm}[t]
\caption{Customization over Homotopic Paths}
\label{alg:overall}
	\LinesNumbered
        \SetKwInOut{Input}{Input}
        \SetKwInOut{Output}{Output}
\Input{$(\mathcal{W}, \{\mathcal{O}_i,i\in \mathcal{M}\})$, $(p_{\texttt{s}}, p_{\texttt{g}})$.}
\Output{$\big\{(\varphi^\star_{\mathcal{F}},\,(\mathcal{F}^\star,\,\mathbf{w}^\star),\,\boldsymbol{\tau}^\star)\big\}$.}
Determine all forest structures~$\mathscr{F}$\;
Set~$\mathscr{S}=\emptyset$\;
\For{$\mathcal{F}\in \mathscr{F}$}{
\tcc{\textbf{Structure Selection}}
 Compute transformation~$\Phi_{\mathcal{F}\rightarrow \mathcal{P}}$
 and~$\mathcal{P}$ by~\eqref{eq:purge}\;
 Compute~$\mathscr{D}^\star$ by~\eqref{eq:all_homo}\;
 \For{$D^\star \in \mathscr{D}^\star$}{
   Compute~$\mathcal{F}^\star$ by~\eqref{eq:optimal-root}\;
   \tcc{\textbf{Weight Optimization}}
   Compute~$\Phi^\star_{\mathcal{F}\rightarrow \mathcal{P}}$
   and~$\mathcal{P}^\star$ by~\eqref{eq:purge}\;
   Optimize~$\mathbf{w}^\star$ by~\eqref{eq:opt-w}\;
   \If{$S(\widetilde{\boldsymbol{\tau}}(t, \mathbf{w}^\star))=S^\star$\label{line:cond}}{
   Compute~$\varphi^\star_{\mathcal{F}}$ and~$\boldsymbol{\tau}^\star$ by~\eqref{eq:complete-nf}\;
   Add~$(\varphi^\star_{\mathcal{F}},\,(\mathcal{F}^\star,\,\mathbf{w}^\star),\,\boldsymbol{\tau}^\star)$
   to~$\mathscr{S}$\;}
 }
}
\end{algorithm}
\vspace{-0.12in}

\subsection{Overall Framework}\label{subsec:overall}
\subsubsection{Customization over Homotopic Paths}\label{subsubsec:preference}
As summarized in Alg.~\ref{alg:overall},
all homotopy classes that can be generated
by the resulting path of valid harmonic potentials within the workspace
is determined by the proposed hybrid optimization scheme.
More specifically, given the workspace configuration
with the start and goal positions,
the set of all forest structures~$\mathscr{F}$ can be determined by iterating
through all choices of roots for each tree-of-stars.
Then, for a chosen structure~$\mathcal{F}$,
its associated transformation~$\Phi_{\mathcal{F}\rightarrow \mathcal{P}}$
and point world~$\mathcal{P}$ are computed,
of which the set of all D-signatures~$\mathscr{D}^\star$ is given by~\eqref{eq:all_homo}.
Afterwards, for each class~$D^\star$, the hybrid optimization scheme
of structure selection via~\eqref{eq:optimal-root}
and weight optimization via~\eqref{eq:opt-w} is applied
to derive the potential~$\varphi^\star_{\mathcal{F}}$,
the parameters~$(\mathcal{F}^\star,\,\mathbf{w}^\star)$,\,
and the resulting path~$\boldsymbol{\tau}^\star$,
which is stored in the set of valid solutions~$\mathscr{S}$.
Note that due to the condition at Line~\ref{line:cond},
only paths with different signs~$S(\cdot)$ in~\eqref{eq:distance-metric}
are added to~$\mathscr{S}$, i.e., they belong to different homotopy classes.

\subsubsection{Complexity Analyses}
\label{subsec:analysis}
The number of all homotopy classes is given by~$2^{M}$,
where $M$ is the number of obstacles.
For a given reference homotopy class,
the complexity to determine the
optimal forest structures is $\mathcal{O}(M N)$,
where $N$ is the maximum depth over all trees~$\{\mathcal{T}_i\}$.
Moreover, the complexity to construct the diffeomorphic transformation
given a forest structure
is~$\mathcal{O}(M^2+M N^2)$.
The parameter optimization has polynomial complexity
as both the objective and constraints are convex.
Lastly, the overall complexity to find all homotopic paths is bounded by
$\mathcal{O}(2^M(M N + M^2+ M N^2))$.


\section{Numerical Experiments} \label{sec:experiments}

\begin{figure}[t]
  \centering
  \includegraphics[width=0.95\hsize]{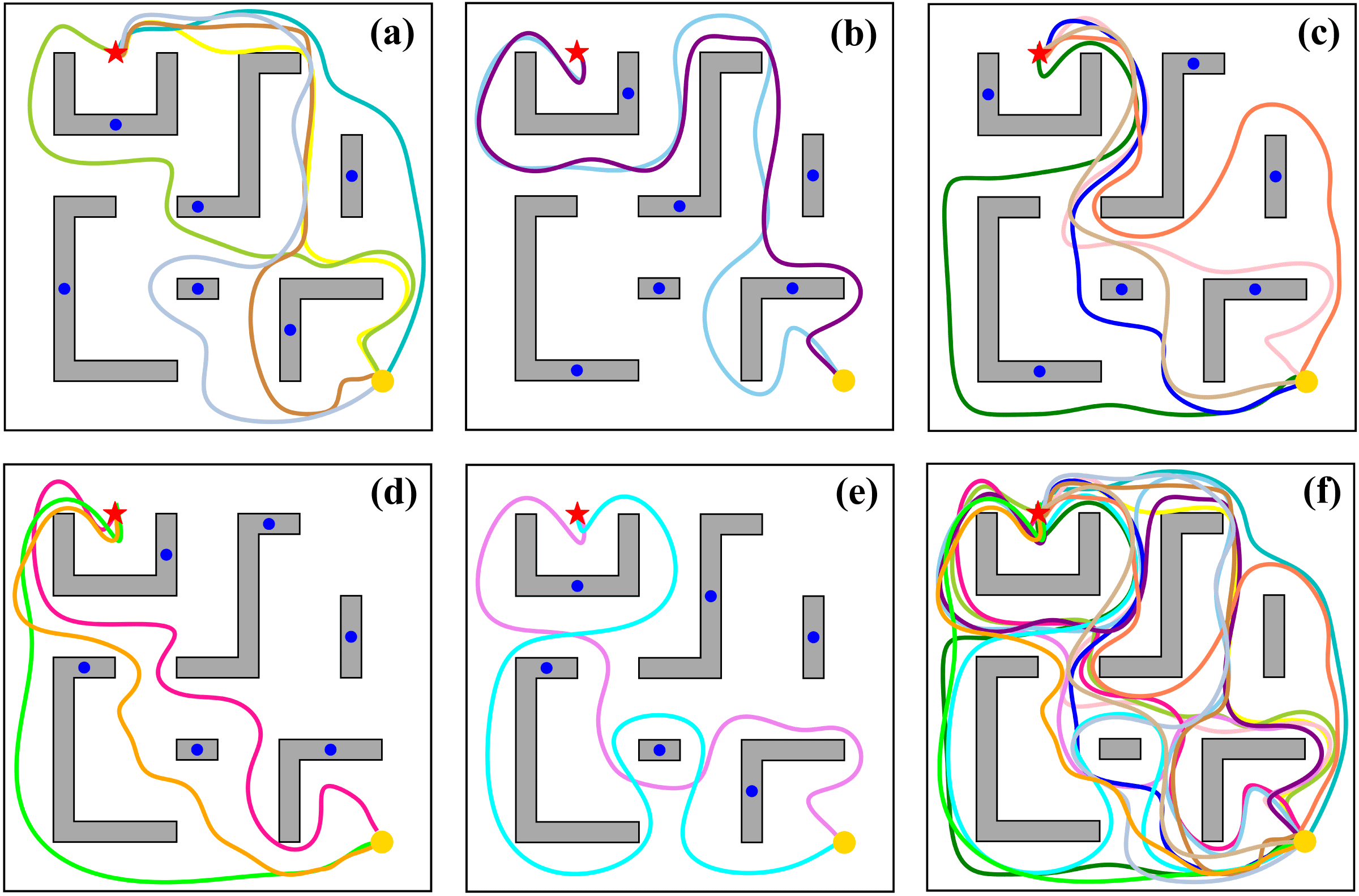}
    \vspace{-0.1in}
  \caption{
    (\textbf{a}-\textbf{e}) Resulting paths of different homotopic classes
    under different weights but the same forest structures (roots in blue).
    (\textbf{f}) paths of~$17$ different homotopic classes found
    in the polygon workspace.}\label{fig:forest_sim}
  \vspace{-0.15in}
\end{figure}


For further validation,
extensive numerical simulations are conducted.
The algorithm is implemented in \texttt{Python3}
and tested on a laptop with an Intel Core i7-1280P CPU.
Simulation videos, experiment videos and source code
  can be found in the supplementary files.

\subsection{System Description}\label{subsec:description}
Two different environments are tested:
(I) a polygon workspace of~$10m \times 10m$ as shown in Fig.~\ref{fig:forest_sim}
featuring~$6$ tree-of-squircles with a maximum depth of~$3$;
(II) an office workspace of~$15m \times 10m$ as shown in Fig.~\ref{fig:hybrid}
featuring~$10$ tree-of-squircles with walls, sofas and tables.
The weight parameter $\mathbf{w}$ in $\phi_{\mathcal{P}}(x,\,\mathbf{w})$
is initialized as $1.0$ for each obstacle~$\mathcal{O}_i$ and~$12.0$ for the goal.
The step size and the lower bounds in~\eqref{eq:opt-w} are set
as~$\delta_k=0.01$ and $\eta = 0.1$.
\subsection{Results}\label{subsec:results}
\subsubsection{Polygon workspace}
To begin with, the hybrid optimization process for a specific homotopy class is investigated.
As shown in Fig.~\ref{fig:gd_process},
the start and goal poses are set to~$(5, 1)$ and $(8, 8)$, respectively.
The desired D-signature is given by $D^\star = (0.19, -0.18, 0.22, -0.39, -0.53, -0.62)$.
To solve~\eqref{eq:optimal-root}, in total~$54$ forest structures are generated,
of which the maximum Fisher distance~$J(D^\star, \mathcal{F}^\star)=1.82$
is found within~$0.12s$ with the forest~$\mathcal{F}^\star$ shown in Fig.~\ref{fig:gd_process}.
Subsequently, the parameter~$\mathbf{w}$ is optimized via~\eqref{eq:opt-w} in $0.47s$.
The intermediate paths in the forest world and point world during different iterations are shown in Fig.~\ref{fig:gd_process},
which transits through various homotopy classes~$S(\cdot)$
from the initial sign~$(1, 1, 1, 1, 1, 1)$ to~$(1, -1, 1, -1, -1, -1)$.
It shows in Fig.~\ref{fig:gd_process} that~$D(\cdot)$ changes discontinuously
at iterations~$79$, $127$, $302$ and $732$,
as the resulting path intersects the neighborhood of saddle points.
The final value of $\mathbf{w}$ is~$(9.56, 0.765, 0.321, 0.727, 1.33, 1.45, 1.08)$.
Moreover, the set of all homotopic paths are generated by Alg.~\ref{alg:overall}.
For the start at $(2.5, 9)$ and goal at $(9, 1)$, the final results
are summarized in Fig.~\ref{fig:forest_sim}.
In total,~$64$ D-signatures are generated as the homotopy classes given by~\eqref{eq:all_homo}
For each homotopy class, the optimal forest structure $\mathcal{F}^\star$
and weight parameter $\mathbf{w}^\star$ are computed, with an average time of~$0.11s$ and $0.4s$.
Finally, a total of $17$ homotopy classes are found within $32.9s$.
It is interesting to note that the homotopic classes generated under different forest
structures are often different,
as clustered in Fig~\ref{fig:forest_sim}.
Additionally, the final harmonic potentials
and the resulting paths are shown in Fig.~\ref{fig:gd_process}.

\subsubsection{Office workspace}
The office workspace consists of four areas that include one restricted region
and three bonus regions as the user preferences, as shown in Fig.~\ref{fig:hybrid}.
Due to the larger number of obstacles,
approximately~$600$ forest structures are evaluated based on
the $2^{10}$ candidate homotopy classes in~$\mathscr{D}$ within~$75.1s$.
In total~$42$ homotopic paths are found via optimizing~$\mathbf{w}$ of dimension~$10$
within~$237.8s$, where each homotopy class takes in average $0.25s$.
Among these paths, $10$ paths that satisfy the preferences
are returned, as shown in Fig.~\ref{fig:hybrid}.
It is worth noting the potentials by
the default choice of forest structure and weights
would result in a path close to the boundary and away from all obstacles.

\subsubsection{Scalability Analysis}
The proposed method is evaluated w.r.t.
the number of trees~$M$ and their depth~$N$, as summarized in Table~\ref{table:scalability-data}.
As $M$ increases from $5$ to $10$ and $20$,
the number of potential homotopy classes would grow exponentially.
The average planning time $T_\texttt{avg}$ for each homotopy class increases
from about~$2.6s$ to~$5.2s$ and~$29.5s$,
with the total number of homotopic paths $|\mathscr{S}|$
being~$9$, $43$ and~$152$, respectively.
Moreover, as $N$ increases to~$4$,
the number of paths roughly doubles due to the increased variety of forest structures,
e.g., $276$ paths for $M=20$ and $N=4$.

 \begin{figure}[t]
   \centering
   \includegraphics[width=0.95\hsize]{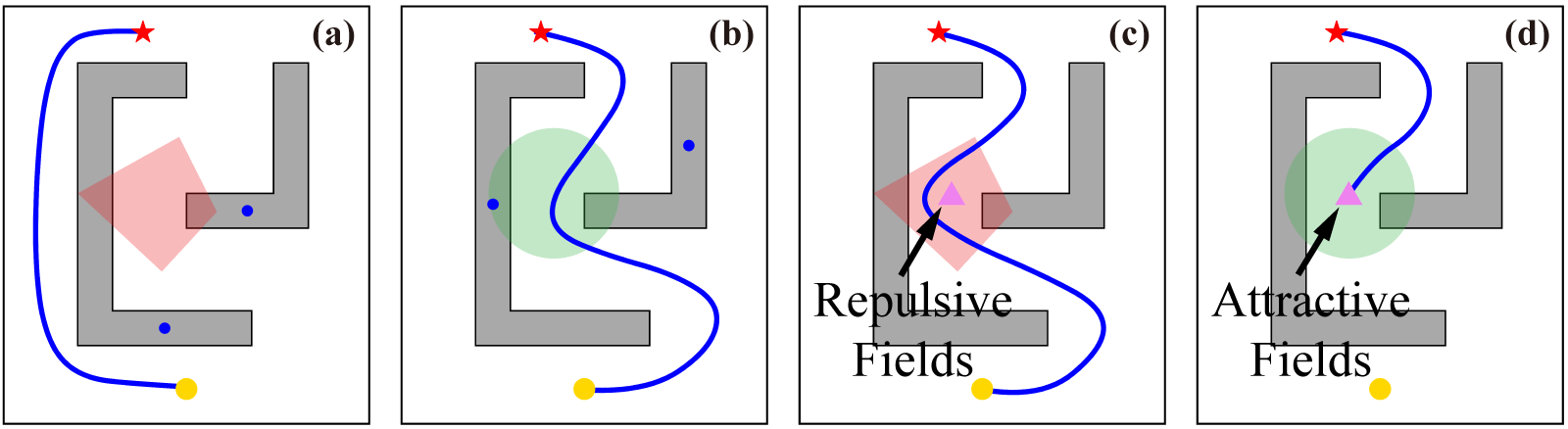}
   \vspace{-0.1in}
   \caption{
   Comparison between the proposed hybrid framework and basic methods that
   apply attractive or repulsive fields directly to regions of interest.
   }\label{fig:compare}
   \vspace{-0.1in}
 \end{figure}

\begin{table}[t]
 \begin{center}
 \begin{threeparttable}
   \caption{Scalability Analysis}\label{table:scalability-data}
   \setlength{\tabcolsep}{0.95\tabcolsep}
   \centering
   \renewcommand{\arraystretch}{1.1}
   \begin{tabular}{c|c c | c c | c c}
     \toprule
     \makecell{$(M, N)$} & (5, 2) & (5, 4) & (10, 2) & (10, 4) & (20, 2) & (20, 4)\\
     \midrule
     \makecell{$|\mathscr{S}|$} & 9 & 17 & 43 & 72 & 152 & 276 \\
     \hline
     \makecell{$T_\texttt{avg}$ [s]} & 2.6 & 1.9 & 5.2 & 3.7 & 29.5 & 21.2 \\
     \bottomrule
   \end{tabular}
 \end{threeparttable}
\end{center}
    \vspace{-0.3in}
\end{table}

\subsection{Comparisons}\label{subsec:comparisons}

 The proposed framework (as \textbf{Ours}) is compared against \emph{five} baselines:
(I) \textbf{Exhaustive}, which enumerates all combinations of weights and forest structures;
(II) \textbf{RRT} from~\cite{lavalle2006planning}, which samples from all possible paths;
(III) \textbf{HM} from~\cite{vlantis2018robot}, which partitions the workspace into $6$ subsets
and selects different intermediate goals;
(IV) \textbf{ReRoot}, which optimizes the forest structure only with the default weights;
(V) \textbf{ParamOpt}, which optimizes weights only with the default forest structure.
A maximum computation time of~$100s$ applies to all methods.
As summarized in Table~\ref{table:comparison-data},
the metrics for comparison include the number of homotopic paths~ $|\mathscr{S}|$,
the total planning time~$T_\texttt{tol}$ to find all homotopy classes,
and the planning time~$T_\texttt{cus}$ to find a specific homotopy class.

Notably, for the same polygon workspace as above,
both \textbf{Exhaustive} and \textbf{RRT} reach the maximum planing time,
yielding $11$ and $13$ paths, respectively,
and requiring approximately $44s$ and $38s$ for a given class.
In contrast, \textbf{HM} constructs the harmonic potential within~$25s$,
yielding $8$ classes via different goal selections,
and an average of $8s$ for a feasible class.
Moreover, \textbf{ReRoot} finds $6$ paths in approximately~$8s$
and the selection of the optimal forest structure only for a given class takes $3.5s$,
while \textbf{ParamOpt} identifies $8$ classes within~$26s$,
taking $8.4s$ to optimize the weights only.
In contrast, our method generates up to $17$ distinct homotopy classes in $32.9s$,
validating the effectiveness of the proposed hybrid optimization scheme,
with around~$5s$ to derive the complete potentials for a desired  class.

The proposed method demonstrates superior performance compared to simple potential-based methods
that apply attractive or repulsive fields directly within regions of interest, as illustrated in Fig.~\ref{fig:compare}.
The repulsive-field approach fails to prevent detours through the avoidance region (Fig.~\ref{fig:compare}(c)),
and the attractive-field approach is attracted to local minima (Fig.~\ref{fig:compare}(d)).
In contrast, the proposed strategy successfully avoids the avoidance region (Fig.~\ref{fig:compare}(a))
and traverses the bonus region while maintaining global convergence (Fig.~\ref{fig:compare}(b)).

 \begin{figure}[t]
   \centering
   \includegraphics[width=0.95\hsize]{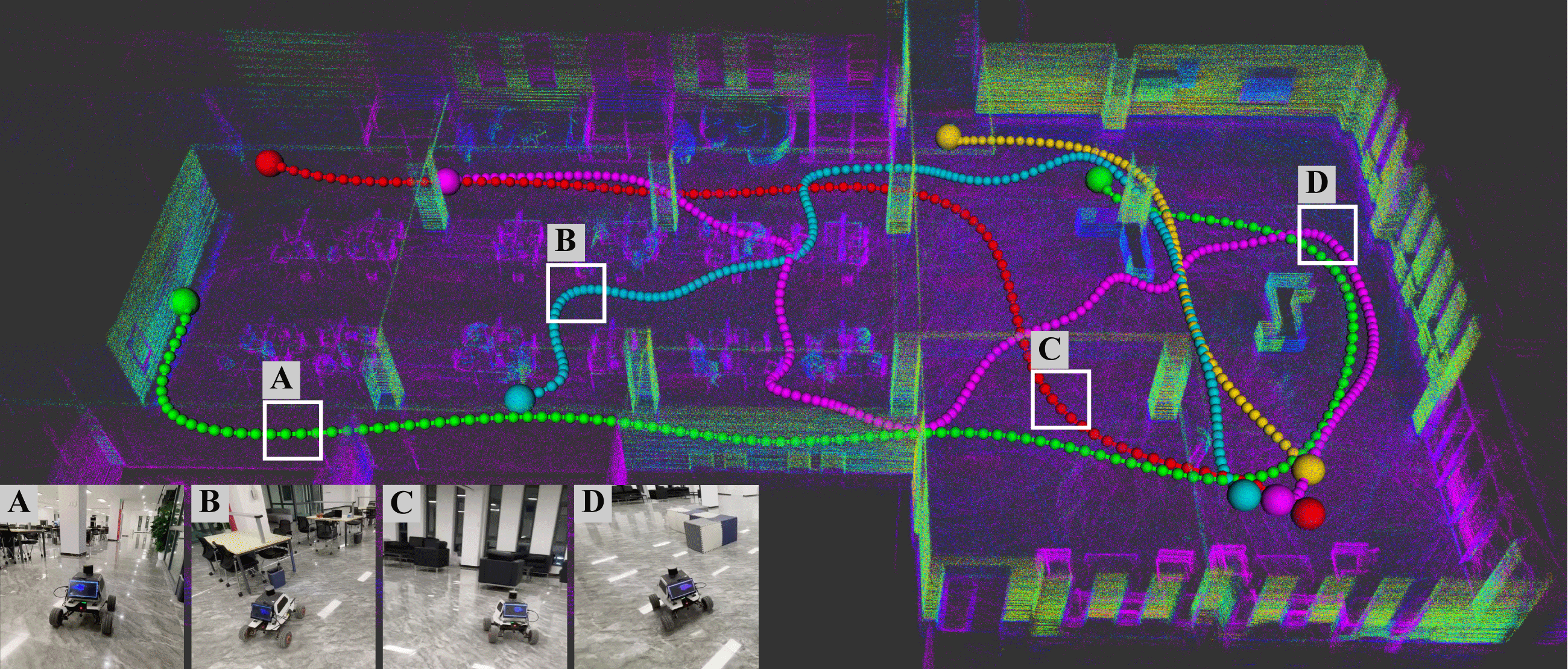}
   \vspace{-0.1in}
   \caption{Final robot trajectories with different homotopy properties,
     given the harmonic potentials found for the
     office environment.}\label{fig:hardware}
   \vspace{-0.1in}
 \end{figure}

\begin{table}[t]
 \begin{center}
 \begin{threeparttable}
   \caption{Comparisons with Baselines}\label{table:comparison-data}
   \vspace{-0.05in}
   \setlength{\tabcolsep}{0.7\tabcolsep}
   \centering
   \renewcommand{\arraystretch}{1.1}
   \begin{tabular}{c|c c c c c c}
     \toprule
     \makecell{Methods} & \textbf{Ours} & \textbf{Exhaustive} & \textbf{RRT} & \textbf{HM} & \textbf{ReRoot} & \textbf{ParamOpt} \\
     \midrule
     \makecell{$|\mathscr{S}|$} & 17 & {11} & {13} & {8} & {6} & {8}\\
     \makecell{$T_\texttt{tol}$ [s]} & 32.9 & {100} & {100} & {25.2} & {7.9} & {26.1} \\
     \makecell{$T_\texttt{cus}$ [s]} & 5.2 & {43.5} & {38.3} & {7.6} & {3.5} & {8.4} \\
     \bottomrule
   \end{tabular}
 \end{threeparttable}
\end{center}
    \vspace{-0.2in}
 \end{table}

\subsection{Hardware Experiments}\label{subsec:experiments}
The proposed method is deployed to a differential-driven robot
operating within an office environment of $40m \times 30m$.
The robot with a radius of $0.35m$ is controlled using the
nonlinear tracking controller from~\cite{wang2025hybrid},
with control gains $k_\upsilon=0.2$ and $k_\omega = 0.5$.
As shown Fig.~\ref{fig:hybrid} and~\ref{fig:hardware},
the robot is equipped with a $360^\circ$ Lidar sensor with a range of $10m$
to create a point cloud map of the environment using SLAM technology
for robot localization.
Based on the accumulated point clouds,
a total of $24$ cluttered obstacles are modeled to construct harmonic potential fields.
Five experimental trials are conducted with distinct user-defined homotopy references.
In each trial, the robot is initialized with random start and goal positions.
Then, harmonic potentials are generated via the proposed framework
given the desired homotopic class for each start-goal configuration,
yielding the optimal $(\mathcal{F}^\star, \mathbf{w}^\star)$ within an average of~$6.7s$.
Lastly, the robot is controlled by tracking the
negated gradient of the associated harmonic fields,
and the resulting trajectories are shown in Fig.~\ref{fig:hardware}.

\section{Conclusion} \label{sec:conclusion}
  This work presents a systematic framework for customizing harmonic potential fields
  to generate paths with desired homotopic properties,
offering a novel solution to task-driven path and motion planning in complex environments.
Via simultaneously optimizing the forest structure and the weighting parameters,
it provides a flexible and customizable planning scheme
for practical robotic applications such as service robots,
search and rescue operations.
Future work involves online optimization within unknown scenes.


\bibliographystyle{IEEEtran}
\bibliography{contents/references}

\end{document}